\documentclass{article}
\usepackage{ijcai21}

\usepackage{times}

\usepackage{soul}
\usepackage{url}
\usepackage[hidelinks]{hyperref}
\usepackage[utf8]{inputenc}
\usepackage[small]{caption}
\usepackage{graphicx}
\usepackage{amsmath}
\usepackage{booktabs}
\title{Synonymy = Translational Equivalence}

\author{
 Bradley Hauer \hspace{5mm} Grzegorz Kondrak\\
 Department of Computing Science\\
 University of Alberta, Edmonton, Canada\\
 {\tt {\{bmhauer,gkondrak\}}@ualberta.ca}
}

\newcommand{\newcite}[1]{\cite{#1}}

\usepackage{url}
\usepackage{amsthm}
\usepackage{amsmath}
\usepackage{centernot}
\usepackage{multirow}
\usepackage{comment}
\usepackage{enumitem}

\usepackage{algpseudocode}
\algtext*{EndWhile}
\algtext*{EndIf}
\algtext*{EndFor}
\algtext*{EndFunction}

\newcommand{\multiwordnet}{multi-wordnet}
\newcommand{\multisynset}{multi-synset}
\newcommand{\multilingual}{multilingual}
\newcommand{\monolingual}{monolingual}
\newcommand{\multisynonymy}{semantic equivalence}
\newcommand{\multisynonyms}{translational equivalents}

\newtheorem{mytheorem}{Theorem}
\newtheorem{mycorollary}{Corollary}

\date{}

\begin{document}

\maketitle

\begin{abstract}
Synonymy and translational equivalence are
the relations of sameness of meaning within and across languages.
As the principal relations in wordnets and multi-wordnets,
they 
are vital to computational lexical semantics,
yet the field suffers from the absence of
a common formal framework
to define their properties and relationship.
This paper 
proposes a unifying treatment of 
these two relations,
which is validated by experiments on existing resources.
In our view,
synonymy and translational equivalence
are simply different types of semantic identity.
The theory 
establishes a solid foundation for 
critically re-evaluating prior work in cross-lingual semantics,
and facilitating
the creation, verification, and amelioration of lexical resources.
\end{abstract}

\section{Introduction}
\label{intro}

Lexical semantics is crucial to natural language understanding (NLU),
identified by \newcite{navigli2018} as a cornerstone of progress for
artificial intelligence. 
{\em Wordnets},
such as the original Princeton WordNet \cite{fellbaum98book},
as well as their multilingual generalizations ({\em multi-wordnets}),
such as BabelNet \cite{navigli2012},
depend on synonymy and translation to 
define the basic units of their ontologies,
called {\em synsets}.
As sources of lexical knowledge,
they are extensively used in
many state-of-the-art NLP systems.
In particular, they serve as the standard sense inventories 
for semantic tasks such as word sense disambiguation (WSD).

The goal of this paper is to 
bolster computational lexical semantics with a theory
that is sound, empirically validated and immediately applicable.
In prior work,
the notions of senses, synsets, and concepts are often confused,
and theoretical assumptions are unstated.
Our theory provides an explanation 
of the relationship between synonymy and translational equivalence,
as well as their role as the basis of wordnets and multi-wordnets.
It also leads to the development of a 
set of best practices for creating multilingual lexical resources, 
which is currently lacking.

We attempt to address several open questions.
Which of the competing \emph{expand} and \emph{merge} paradigms 
should be applied for multi-wordnet construction?
Can fine-grained senses be clustered 
while preserving fundamental properties of synsets?
How can synonymy be maintained
when extending synsets in multilingual settings?
Can effective error detection algorithms for
automatically-constructed lexical resources
be derived from sound theoretical foundations?

Our main contribution is
a clear and consistent theoretical framework 
for reasoning about senses, concepts, and translations. 
Building on
a set of clearly formulated axioms,
we formulate and prove several theorems
that characterize the relationship 
between synonymy and translational equivalence
at the level of both words and senses.
These results 
allow us to reassess previous methods,
and explore their consequences and implications,
which lead towards resolving open issues. 
While some of these propositions may reflect unstated intuitions
discernible in prior work,
their explicit statement and derivation from first principles 
constitutes a novel contribution.

Our work offers practical benefits to the research community.
We provide experimental evidence 
for the validity of our theory.
Analysis of the apparent exceptions to our theorems shows that 
most of them are due to errors in lexical resources,
which are often caused by versioning issues.
This leads us to propose 
an algorithm that can not only flag such errors but also correct them.
The algorithm 
provides an immediate application of our theoretical results,
and represents a step towards improvement of existing resources, 
as well as 
the creation of new resources in an automated fashion.

Finally, we show that our theory
implies important consequences for lexical semantics.
Since word senses
are determined by word synonymy,
sense granularity cannot be substantially reduced 
without violating the fundamental properties of wordnets.
The expand model of multi-wordnet construction 
has the potential of preserving those properties,
but at the cost of
increased sense granularity.
The most surprising finding is that
the existence of an exact matching between synsets across wordnets implies 
the universality of lexicalized concepts in natural languages.

This paper has the following structure: 
In Section \ref{prelim},
we provide precise definitions of the basic terms and assumptions.
In Section \ref{theorems},
we formulate and prove several theorems and corollaries.
In Section \ref{exp},
we describe our validation experiments
and propose an error correction algorithm.
In Section \ref{cu},
we discuss the implications of our theory for multilingual semantics.

\section{Semantic Equivalence}
\label{prelim}

In this
section,
we define
the theoretical properties of wordnets and multi-wordnets,
and propose a unified treatment of
synonymy and translational equivalence.
The properties,
which follow from the basic definitions and assumptions 
in the original WordNet,
are often implicitly assumed in prior work,
but have never been precisely formulated.
We view existing lexical resources
as imperfect approximations of theoretical models.
The divergence of contemporary resources from a hypothetical ideal 
does not preclude theoretical analysis of lexical semantics.

\subsection{Synonymy}
\label{synonymy}

Synonymy, 
the relation of {\rm sameness of meaning},
can be established by a {\em substitution test}: 
two linguistic expressions (e.g. words)
are considered synonymous
if and only if they
can be substituted for one another
in a sentence 
without changing its meaning\footnote{We disregard 
distinctions in register, stylistic constraints,
frequency, or other factors which do not directly affect meaning.
Substitution may require some grammatical and/or morphological adjustment.}
\cite{murphy10}.
{\em Absolute synonyms} 
can be substituted for one another 
without a change in meaning
in any context,
whereas {\em near-synonyms} are interchangeable in some
contexts.\footnote{This formalization of near-synonymy 
is more precise than definitions based on
the vague notion of word similarity.}
For example, 
since substituting the word {\em gist} with the word {\em essence}
does not change the meaning of the sentence
{\em ``We understand the gist of the argument,''}
the two words
are near-synonyms.
Considered as relations,
both absolute synonymy and near-synonymy are reflexive and symmetric,
but only the former is transitive
(e.g., consider the word triple {\em accusation}, {\em charge}, and {\em cost}).
As a consequence, absolute synonymy is an equivalence relation,
which partitions expressions into semantic equivalence classes. 
In this paper,
the term {\em synonymy} by itself refers to absolute synonymy.

\subsection{Word Senses}
\label{senses}

Although 
synonymy can be defined on various types of linguistic units,
including sentences and phrases
({\em paraphrases}),
our main focus is on words and their senses.
We assume that 
every content word token
has a particular meaning,
and define a {\em word sense} (or simply {\em sense})
as a partition\footnote{For example, 
99.7\% of sense-annotated tokens in SemCor
are assigned a single WordNet sense.}
of these meanings
\cite{kilgarriff1997}.
It follows 
that every content word token is used in exactly one sense,
and that every content word type
has at least one sense.
Words are either {\em monosemous} or {\em polysemous}
depending on whether they have only one or multiple senses.

\subsection{Synsets}
\label{synsets}

A wordnet is a lexical ontology 
in which words 
(including non-compositional phrases, such as `single out')
are organized into {\em synsets}.
A synset is a set of words that are interchangeable in some context
\cite{fellbaum98book}.
Therefore, words which share a synset must be either absolute or near-synonyms.
Each word in a synset 
can be used to express 
a common
{\em lexicalized concept}
\cite{miller1995acm}.
For example, 
a synset that contains the nouns 
{\em gist}, {\em essence}, and {\em core}
represents a single concept. 
The concept,
which each of these near-synonyms can express,
is defined by 
the set of contexts in which these words are interchangeable.

Synsets provide another way of defining a word sense,
namely as a {\em(concept, word)} tuple.
In Table~\ref{sense_examples}
(adapted from \newcite{miller1990lexi}),
columns correspond to words, rows correspond to concepts (or synsets), 
and each non-empty cell is a word sense.
For example, 
$e_1 = \mbox{\it earth}$,
$e_2 = \mbox{\it ground}$,
$e_3 = \mbox{\it reason}$,
$f_1 = \mbox{\it terra}$,
$f_2 = \mbox{\it motivo}$.
Each concept is lexicalized by at least one word,
and each word lexicalizes at least one concept.
A single word sense always represents the same concept. 
The number of senses of each content word
is equal to the number of concepts that it lexicalizes.
Thus,
synsets can be equivalently defined
as either sets of words
or sets of unique word senses.
Furthermore,
words that share no synsets can never be synonymous.

\begin{table}[t]
\centering
\small
\begin{tabular}{c|cccc|ccc|}
\multicolumn{1}{c}{} 
 & \multicolumn{4}{c}{Language $E$} & \multicolumn{3}{c}{Language $F$}\\
\cline{2-8}
 & $e_1$ & $e_2$ & $e_3$ & \ldots & $f_1$ & $f_2$ & \dots  \\
\cline{2-8}
$C_1$ & $s_{1,1}$ &              &              & & $t_{1,1}$ & & \\ 
$C_2$ & $s_{1,2}$ & $s_{2,2}$ &              & & $t_{1,2}$ &  & \\ 
$C_3$ &             & $s_{2,3}$ & $s_{3,3}$ & & & $t_{2,3}$ & \\
\ldots & & & & \ldots & & & \ldots \\
\cline{2-8}
\end{tabular}
\caption{
Word senses as the intersection of words and
concepts.}
\nocite{miller1990lexi}
\label{sense_examples}
\end{table}

Consider the relation between two senses
that holds if
and only if 
the senses share a synset.
The reflexivity, symmetry, and transitivity of the relation
follow directly from the definition of a synset,
which is based on the substitution test.
Since the senses
represent the same lexicalized concept,
this equivalence relation represents
{\em absolute synonymy of senses}.
For example, since
the senses {\rm gist\#n\#2} and {\rm essence\#n\#1}
are absolute synonyms,
the word {\em gist} 
can always be replaced by the word {\em essence}
when it is used in sense {\rm gist\#n\#2}. 
Different from words, 
synonymy of senses in a wordnet is always absolute.
Therefore,
synsets can be viewed as
{\em the equivalence classes} of the relation of absolute synonymy of senses.
This novel proposition,
which we refer to as {\em the wordnet assumption},
is one of the foundations of our theory.

We define five {\em synset properties} 
which follow
from the preceding definitions and assumptions,
and which must be maintained in wordnets:
Although these properties are often implicitly assumed,
to the best of our knowledge,
they have never been explicitly stated.
We use them to construct the proofs in Section~\ref{theorems}.
\setlist{nolistsep}
\begin{enumerate}[noitemsep]
\item \label{pMono}
{\em A word is monosemous iff it is 
in a single synset.}
{\em A word is polysemous iff it is 
in multiple synsets.}
\item \label{pShare}
{\em Words are near-synonyms iff
they share at least one synset.}
{\em Words are absolute synonyms iff
they share all their synsets.}
\item \label{pSame}
{\em Word senses are synonymous iff
they are in the same synset.}
\item \label{pOne}
{\em Every word sense
belongs to exactly one synset.}
\item \label{pEvery}
{\em Every sense of a polysemous word belongs to a different synset.}
\end{enumerate}

\subsection{Translational Equivalence}
\label{te}

Having defined synonymy, wordnets, and synsets in the {\monolingual} setting,
we are now in a position to extend these notions to the {\multilingual} setting.
The cross-lingual analogue of
synonymy is translational equivalence,
which
is the relation of
sameness of meaning between expressions in distinct languages
\cite{uresova2018}.
Translational equivalence
can be established by a
{\em translation test}: two expressions in distinct languages are 
absolute 
translational equivalents
if and only if 
each can be translated into the other
in any context.

We postulate that 
the relations of 
synonymy and translational equivalence
can be combined via a simple union operation 
to produce a single relation of \emph{\multisynonymy},
which is applicable to any pair of expressions 
in the same or different natural languages.

\subsection{Multi-Synsets}
\label{multisynsets}

The notion of trans-lingual {\multisynonymy}
is fundamental to
{\multilingual} semantic networks, or {\em {\multiwordnet}s}, 
such as BabelNet 
\cite{navigli2012}.
Just as wordnets are comprised of inter-connected synsets,
the basic units of multi-wordnets
are {\multilingual} synsets, 
which we refer to
as {\em {\multisynset}s}.
Multi-wordnets and multi-synsets are extensions of the corresponding
{\monolingual} notions to the {\multilingual} setting.
In particular, {\multisynset}s
contain words 
in one or more languages
that express the same concept
\cite{camacho2015acl}.
For example,
BabelNet multi-synsets are populated by
translations of word senses that represent a given concept
\cite{navigli2010}.

Lexical gaps occur
when a concept is lexicalized in some but not all natural languages.  
A compositional phrase or a hypernym 
can be employed to translate such a concept \cite{rudnicka2012}.
In multi-wordnets,
lexical gaps can be represented by
special tokens within multi-synsets
 \cite{bentivogli2003}.

There are two principal approaches to the construction of
{\multilingual} wordnets 
\cite{vossen1996}.
The \emph{expand model} 
uses a {\monolingual} ``pivot'' wordnet
(typically the Princeton WordNet)
to establish a base set of concepts and relations,
to which 
words or synsets
in other languages are then linked.
The \emph{merge model} 
attempts to link the synsets
of independently constructed {\monolingual} wordnets
using a pre-defined set of inter-lingual relations.

We posit that 
multi-synsets should maintain the properties of {\monolingual} 
synsets.\footnote{As postulated by \newcite{kwong2018}.}
If this postulate
is satisfied,
{\monolingual} synsets
can be obtained from 
multi-synsets by simply 
restricting them to
a given individual language.
Furthermore, we posit that 
{\rm words} from distinct languages share a multi-synset 
if and only if they are mutual translations in some context.
Since the senses that share multi-synsets 
represent the same lexicalized concept,
they are {\em absolute {\multisynonyms}}.
Therefore, 
multi-synsets can be viewed as
{\em the equivalence classes} of the relation of {\multisynonymy} 
between senses
within and across languages.
This novel proposition,
which we refer to as {\em the multi-wordnet assumption},
is the second pillar of our theory.

\section{Theorems}
\label{theorems}

Having established our terminology and assumptions,
we now proceed to present our theoretical results.
Each of the following four subsections presents a novel theorem
in lexical semantics.

\subsection{Synonymy and Translation of Senses}
\label{sense_theorems}

We first present 
our principal theorem and two corollaries
which establish the relationship 
between synonymy and translational equivalence
at the level of senses.
Our notation follows the example in Table~\ref{sense_examples}.
We use different base letters for distinct languages:
$s$ vs.\  $t$ for senses, 
$e$ vs.\  $f$ for words,
and
$E$ vs.\  $F$ for languages.
Subscripts distinguish between senses and words within the same language.
The predicates {\em syn}$(\cdot,\cdot)$ and {\em tr}$(\cdot,\cdot)$
express the propositions that 
two expressions (senses or words) 
are absolute synonyms or {\multisynonyms}, respectively.

\begin{mytheorem}
\label{four-senses}
Given two pairs of word senses $(s_x,t_u)$
and $(s_y,t_v)$ such that both pairs are {\multisynonyms}:
$s_x$ and $s_y$ are synonymous 
{\bf if and only if}
$t_u$ and $t_v$ are synonymous.
\end{mytheorem}
\begin{proof}
By synset property \#\ref{pOne},
every sense belongs to exactly one multi-synset.
By synset property \#\ref{pSame},
if two word senses are synonymous,
they must be in the same multi-synset.
By the multi-wordnet assumption,
$s_x$ must share a multi-synset with $t_u$,
and $s_y$ must share a multi-synset with $t_v$.
Therefore,
if either $s_x$ and $s_y$ or $t_u$ and $t_v$ are synonymous,
all four senses
must belong to the same multi-synset,
which implies that they are semantically equivalent.
\end{proof}

Both of the following corollaries 
differ from Theorem \ref{four-senses} in that 
they involve triples of senses, rather than quadruples.
The first corollary,
which can be viewed as a special case of Theorem~\ref{four-senses}, 
states that 
{\rm senses that translate into the same foreign sense must be synonymous.}
This observation could lead to an algorithmic method for
constructing or augmenting synsets using sense-annotated bitexts.

\begin{mycorollary}
\label{theorem-tis}
\textbf{Translational Equivalence of Senses Implies Synonymy:}
$ \forall s_x,s_y \in E: 
\forall t_z \in F:
tr(s_x,t_z) \wedge tr(s_y,t_z) \Rightarrow syn(s_x,s_y). $ 
\end{mycorollary}

The second corollary establishes the reverse implication ---
{\rm all senses that are synonymous must translate into the same sense,
provided that a single-word translation exists in the other language.} 

\begin{mycorollary}
\label{sense-sit}
\textbf{Synonymy of Senses Implies Translational Equivalence:}
$ \forall s_x \in E, s_y \in E, t_z \in F: 
syn(s_x,s_y) \wedge tr(s_y,t_z)
\Rightarrow tr(s_x,t_z). $ 
\end{mycorollary}

Both corollaries,
as well as Theorem~\ref{four-senses} itself,
follow
from the transitivity of the relation of {\multisynonymy}.

\subsection{Synonymy and Translation of Words}
\label{word_theorems}

\newcite{yao2012} observe that 
prior work,
such as \newcite{gale1992} and \newcite{diab2002},
is based on one of the two ``alternate'' assumptions,
which have the same antecedent but different consequents:

\begin{it}
{\em Antecedent:} 
Two different words 
$e_x$ and $e_y$ in language $E$
are aligned to the same word 
$f_z$ in language $F$.

{\em Consequents:}
\begin{enumerate}
\item 
$f_z$ is polysemous
{\em (``polysemy assumption'')}
\item 
$e_x$ and $e_y$ are synonymous
{\em (``synonymy assumption'')}
\end{enumerate}
\end{it}

\newcite{yao2012} 
perform experiments 
on two bilingual corpora,
using a lexical sample of 50 words from OntoNotes \cite{hovy2006},
and conclude that neither assumption holds significantly more often
than the other.
However, they stop short of proposing a principled solution to the problem.

According to our theory,
neither of the two assumptions need hold universally.
For example,
although both
{\em time} and
{\em weather}
are translations of the Italian word
{\em tempo},
it would be wrong to conclude that the two English words 
are near-synonyms.
This is because,
unlike absolute synonymy of senses,
near-synonymy of words is not transitive
in either {\monolingual} or {\multilingual} setting.
On the other hand, 
although both
{\em bundle} and
{\em package}
are translations of the Italian
{\em involto},
this does not imply that
the Italian word is polysemous;
indeed, 
both English words 
translate a single sense of {\em involto}.

We postulate that 
the polysemy and synonymy assumptions
can be integrated
into a single theorem.
In fact, 
the two consequents are not exclusive;
for example,
{\em test} and
{\em trial},
which are {\bf synonymous},
are both translations of
Italian 
{\em prova},
which is {\bf polysemous}.
Thus,
the theorem entails
a {\em non-exclusive union}
of the two consequents:
\begin{mytheorem}
\label{word-theorem}
Given two words $e_x$ and $e_y$ in language $E$
and a word $f_z$ in language $F$:
{\bf if} 
$e_x$ and $e_y$ are both translations of $f_z$
{\bf then}
$e_x$ and $e_y$ are near-synonymous
{\bf or}
$f_z$ is polysemous.
\end{mytheorem}
\begin{proof}
If $f_z$ is polysemous, 
the implication holds trivially.
Otherwise, $f_z$ must be monosemous, 
so by synset property \#\ref{pMono},
there exists only one multi-synset that contains $f_z$.
By the multi-synset property,
both $e_x$ and $e_y$ must share a multi-synset with $f_z$.
Therefore, 
by synset property \#\ref{pShare},
since $e_x$ and $e_y$ share a synset,
$e_x$ and $e_y$ are near-synonyms.
\end{proof}

In conclusion,
our theory demonstrates that 
systems which are based exclusively on one of the two assumptions,
such as \newcite{bannard05} and \newcite{lefever2011},
fail to consider a substantial number of relevant instances.
Theorem \ref{word-theorem}
provides a more reliable foundation, 
which we validate empirically in Section \ref{exp2}.

\subsection{Absolute Synonymy of Words}

\newcite{yao2012} use the term synonymy
to mean 
near-synonymy.
What does our theory predict 
if synonymy of words is taken to mean absolute synonymy instead?
It turns out that exactly one of the two assumptions,
the synonymy assumption,
holds universally.

In Section~\ref{sense_theorems},
we formulated Theorem~\ref{four-senses} and its two corollaries 
to characterize the relation between 
absolute synonymy and translation of {\em senses}.
We can formulate analogous results
to characterize the relation between 
absolute synonymy and translation of {\em words}.

\begin{mytheorem}
\label{four-words}
Given two pairs of words $(e_x,f_u)$
and $(e_y,f_v)$ that are absolute {\multisynonyms}:
$e_x$ and $e_y$ are absolute synonyms 
{\bf if and only if}
$f_u$ and $f_v$ are absolute synonyms.
\end{mytheorem}

\begin{proof}
By synset property \#\ref{pEvery},
every sense of a given word belongs to a different synset.
By synset property \#\ref{pShare},
absolute synonyms share all their synsets.
By the multi-wordnet assumption,
$e_x$ must share all its multi-synsets with $f_u$,
and $e_y$ must share all its multi-synsets with $f_v$.
Therefore,
if either $e_x$ and $e_y$ or $f_u$ and $f_v$ are absolute synonyms,
all four words
must share all their multi-synsets,
which implies that they are semantically equivalent.
\end{proof}

Just like Theorem~\ref{four-senses},
Theorem~\ref{four-words} implies two corollaries.
First,
if two different words can {\em always} be translated
by the same foreign word (and vice-versa),
then the two words are absolute synonyms.
Second, 
the sets of translations of absolute synonyms must be identical.
We omit the formal statements of the two corollaries,
as they are almost identical to 
Corollaries~\ref{theorem-tis} and~\ref{sense-sit} in
Section~\ref{sense_theorems}.

\subsection{Translations of Near-Synonyms}

Our final theorem 
can be viewed as the converse of the synonymy assumption from
Section~\ref{word_theorems}.
Theorem~\ref{sit-words} states that 
near-synonymy implies the existence of a shared translation.
Since the theorem
needs to account for lexical gaps, 
it employs the term {\em phrase},
which may be either a single word or a
sequence of words.

\begin{mytheorem}
\label{sit-words}
Given two words $e_x$ and $e_y$ in language $E$:
{\bf if} $e_x$ and $e_y$ are near-synonyms
{\bf then} there exists an phrase $\varphi$ in language $F$
such that both $e_x$ and $e_y$ can be translated by $\varphi$.
\end{mytheorem}
\begin{proof}
Since $e_x$ and $e_y$ are near-synonyms,
there exists a multi-synset $M$ that they share.
By the substitution test,
there must exist a pair of sentences $S_1$ and $S_2$
that have the same meaning, 
and differ only in containing either $e_x$ for $e_y$ at the same position.
Since $S_1$ and $S_2$ have the same meaning,
they can both be translated by the same sentence $T$ in language $F$.
The sequence of one or more words within $T$
that translate $e_x$ and $e_y$ in $S_1$ and $S_2$, respectively,
constitutes the phrase $\varphi$.
\end{proof}

As a corollary,
if the concept that corresponds to the multi-synset $M$ 
is {\em lexicalized} in language $F$
(i.e., there is no corresponding lexical gap in $F$),
then there exists a word $f_z$ that can translate both $e_x$ and $e_y$.

\section{Experimental Evidence}
\label{exp}

In this section, we describe experiments that test the
predictions of our theory,
and demonstrate how our theory can be used
to automatically detect and correct errors in semantic resources.

\subsection{Methodology}
\label{methodology}

Our methodology is based on cross-checking 
the evidence for synonymy and translational equivalence
between different semantic resources.
We empirically validate our theorems
on a sense-annotated word-aligned parallel corpus ({\em bitext})
coupled with a multi-wordnet that covers the two languages of the bitext.
We operationalize the relations of synonymy and translational equivalence
on senses as follows:
(a) senses 
are synonymous
if they are associated with the same multi-synset,
and (b)
senses 
are translational equivalents
if a pair of words annotated with those senses are aligned 
in the bitext.\footnote{This is an operational simplification
because not all words that are aligned in a bitext
are necessarily dictionary translations.
In non-literal prose translation,
the relationship between senses may instead be that of lexical entailment
(e.g. English {\em plant} translated as Italian {\em fiore} ``flower.'')}

Not all predictions of our theory can be tested in this way.
An alignment link in a sense-annotated bitext demonstrates that 
two words or senses can be translated into each other.
However, 
since no bitext, regardless of its size,
can be guaranteed to contain all possible translations,
it cannot furnish conclusive proof that two words 
are translational equivalents,
or that they both can be translated into the same foreign word.
Similarly, no bilingual dictionaries or thesauri 
include exhaustive lists of all translations and synonyms.
However, 
all of our theorems and corollaries 
are proven using the same theory of sense, synonymy, and translation.
Therefore, 
the empirical evidence that we present 
for Theorems \ref{four-senses} and \ref{word-theorem}
and Corollary~\ref{theorem-tis}
provides indirect support for 
Theorems \ref{four-words} and \ref{sit-words}
and Corollary~\ref{sense-sit}.

\subsection{Resources}
\label{experiment-msc}

The multi-wordnet we use in our experiments is
MultiWordNet\footnote{\it{http://multiwordnet.fbk.eu}}
(MWN) version 1.5.0
\cite{pianta2002}.
We chose MWN because of its superior quality and coverage,
in comparison to other multi-wordnets,
such as BabelNet or Open Multilingual Wordnet \cite{bond2013}.
Each multi-synset in MWN
is associated 
with a part of speech and a unique identifier,
while each English lemma
is labeled with the corresponding WordNet 1.6 sense number
in each of its synsets.

As our word-aligned sense-annotated bitext,
we use
MultiSemCor\footnote{\it{http://multisemcor.fbk.eu/index.php}}
(MSC) version 1.1
\cite{bentivogli2005},
a sense-annotated English-Italian bitext
crafted semi-automatically by bilingual lexicographers
using a professional translation of SemCor 
\cite{miller1993}.
We exclude
English tokens that are annotated with multiple sense numbers,
as well as instances with missing sense annotations
or mismatched POS.
Together, these criteria apply to approximately 1\% of all instances,
which leaves 91,438 aligned English-Italian sense pairs.

\subsection{Absolute Synonymy and Translation}

Absolute synonymy of words is considered rare,
to the point that its very existence is denied \cite{jm2}.
By synset property \#\ref{pShare},
words that share all their synsets are absolute synonyms.
According to this criterion,
69,775 words            
in Princeton WordNet 3.0
have at least one absolute synonym.
They include
variant spellings, such as {\em liter} and {\em litre},
variant terminology, such as {\em atmometer} and {\em evaporometer},
and abbreviations, such as {\em kg} and {\em kilogram}.
While most of these words are monosemous,
there are absolute synonyms that share as many as eight synsets;
For example, the nouns \emph{haste} and \emph{hurry}
are absolute synonyms with three senses each.

The {\multilingual} extension of absolute synonymy
is believed to be similarly rare \cite{uresova2018}.
Yet, we find that
MultiWordNet contains 45,717 English-Italian word pairs
which appear in exactly the same synsets,
indicating that one can always translate the other.
Many of these absolute {\multisynonyms} are cognates, 
such as {\em globally} and {\em globalmente},
and borrowings,
such as {\em internet}.

\subsection{Word-Level Verification}
\label{exp2}

We test Theorem~\ref{word-theorem} (Section~\ref{sense_theorems})
on MSC by identifying all triples 
that consist of two different English words and
an Italian word that they are both aligned to at least once in MSC,
e.g., {\em (inverse, opposite, contrario)}.
The theorem implies that in each such case
the two English words are near-synonymous or the Italian word is polysemous.
We find that among 17,272 distinct triples,
17,136 include a polysemous Italian word,  
3,343 contain a pair of English near-synonyms,
and 3,207 involve both polysemy and synonymy.
This shows that in MSC the polysemy assumption 
holds substantially more often than the synonymy assumption,
which differs from the conclusions of \newcite{yao2012}
(Section~\ref{word_theorems}).
We attribute this discrepancy 
to their use of a coarse-grained OntoNotes sense inventory,
as well as testing relatively small lexical samples,
rather than entire lexicons.
Since no exceptions to Theorem~\ref{word-theorem} are found,
we conclude that the experiment fully supports 
its validity.

\subsection{Sense-Level Verification}
\label{exp1}

In Section \ref{te}, we posited that
the relations of synonymy and translational equivalence
could be viewed as the intra- and inter-lingual components
of a single relation of semantic equivalence,
applicable to any pair of expressions
in identical or different languages.
MWN and MSC, which were developed independently,
allow us to empirically test 
this foundational postulate of our theory
by checking whether all senses
that are aligned in the bitext
do indeed share a multi-synset. 

We find that only 18 of the aligned sense pairs
appear to violate the synonymy constraint.
10 of these instances 
involve Italian lemmas annotated with synsets which do not contain them,
which implies either an omission in MWN or a translation error.
Further analysis shows that
the remaining 8 instances are due to word alignment errors in MSC.

We conclude that 
the annotations in MWN and MSC
fully conform to the prediction of our theory
that translational equivalence of senses implies their synonymy
(Section~\ref{te}).
Furthermore, the experiment suggests that 
any apparent exceptions to this prediction 
may indicate sense annotation errors.

\subsection{Upgrading MultiSemCor}
\label{upgrade}

The sense-level experiment in previous section
identified only a small number of errors
because of the high-quality of manually constructed resources.
However, many lexical resources 
are created and updated with automated procedures,
resulting in less reliable annotations.
In this section, we aim to demonstrate that our theory 
can be used to detect a substantial number of errors in a noisy resource,
paving the way to a radical improvement of its quality.

The non-trivial task that we focus on is
updating sense annotations to a new version of WordNet.
The importance of this task is demonstrated by 
previously reported attempts.
\newcite{daude2003} provide probabilistic 
mappings between senses in different WordNet versions.
\newcite{bond2013} simply take the most probable 
mapping for each synset,
while \newcite{raganato2017} manually correct annotations that
cannot be confidently mapped.
Specifically,
we show that our theorems can identify 
MSC sense annotations that are correct under WordNet 1.6
but incorrect under WordNet 3.0.
For the remainder of this section,
WordNet 3.0 is used to determine English word and sense synonymy.

We apply Theorem \ref{four-senses} and Corollary \ref{theorem-tis},
in both translation directions.
In each of the four sense-level experiments,
we identify in the annotated bitext all unique instances
that satisfy the premise of the proposition that is being tested.
For Corollary~\ref{theorem-tis},
the instances are sense {\em triples}
that consist of pairs of source senses that are aligned with 
the same target sense.
For Theorem \ref{four-senses},
the instances are sense {\em quadruples}
that consist of pairs of source senses that are aligned with 
two distinct but synonymous target senses.
The two source senses must be distinct, but they may belong to the same word.
Finally, we verify whether the two source senses are synonymous,
as predicted by our theory.

The results of the experiments are summarized in Table \ref{table-thm1}.
The first row shows the number of unique instances found in MSC,
while the second row shows what number of those instances 
appear to contradict our theory.
Is each of these apparent exceptions to 
Corollary \ref{theorem-tis} and Theorem \ref{four-senses} 
an indication of an out-of-date sense annotation?
Or are some of them actual exceptions to our theory?
In the following two sections,
we analyze samples of the exceptions
in order to answer these questions.

\subsection{Manual Exception Analysis}
\label{analysis}

For each of Corollary \ref{theorem-tis}
and Theorem \ref{four-senses},
we randomly selected 25 of the apparent exceptions
in the {\em en}$\rightarrow${\em it} direction.
Each of these exceptions 
consists of two English-Italian sense alignment pairs,
which involve Italian senses 
that are either identical  (for Corollary \ref{theorem-tis}),
or distinct but synonymous (for Theorem \ref{four-senses}).
For each exception, we manually analyze a sample of 
up to ten of the corresponding English sentences from the bitext.
We consult WordNet 3.0 synsets, glosses, and usage examples
to make judgments on the correctness of the annotations.

We find that all 50 apparent exceptions,
as our theory would predict,
include a sense annotation which is not correct under WordNet 3.0.
Specifically,
37 instances involve out-of-date WordNet 1.6 sense numbers,
11 instances involve senses new to WordNet 3.0,
1 instance involves both of these issues,
and 1 instance is an annotation error on the Italian side.
We interpret these findings
as very strong support for the soundness of our theory.

\begin{table}[t]
\centering
\begin{small}
\begin{tabular}{c|c|c|c|c}
            & \multicolumn{2}{|c|}{Corollary \ref{theorem-tis}} &
\multicolumn{2}{|c}{Theorem \ref{four-senses}} \\
\hline
            & en$\rightarrow$it & it$\rightarrow$en &
en$\rightarrow$it & it$\rightarrow$en \\
\hline
Instances   &  1792 & 10597 & 19080 & 21689 \\
Exceptions  &   194 &  1069 &  1965 &  3298 \\
\end{tabular}
\end{small}
\caption{The results of the MSC error detection experiment.}
\label{table-thm1}
\vspace{-2mm}
\end{table}

\subsection{Substitution Test Experiment}
\label{hip}

In order to extend the scope of our analysis,
we performed an annotation experiment
based on the substitution test for synonymy (Section \ref{synonymy}).
The rationale of the experiment is that
a sense annotation must be incorrect
if substituting it with another sense from the same synset
either changes 
the meaning of the sentence
or renders it meaningless.

We identified 77 of the 194 exceptions to Corollary~\ref{theorem-tis},
in the {\em en}$\rightarrow${\em it} direction,
such that 
exactly one of the two English senses
shares a synset 
with a sense of the word of the other English sense.
We then created a set of 77 English sentence pairs that
differ only in the word in question.
The original sentence
for each exception
is randomly selected from 
the set of the sentences in MSC that correspond to 
the exception.
In the modified sentence, the word annotated with the first sense is replaced 
with the word of the second sense from the same synset.
For example,
{\em Their world turned black}
is modified to
{\em Their world reversed black}.

We asked two native English speakers to decide independently
whether the original and modified sentences had the same meaning.
In 82\% of the cases, 
the annotators judged that the meaning was not preserved,
which implies that the sense annotation in MSC is incorrect.
Furthermore, 
our manual analysis based on 
the contents, glosses and usage examples of the synsets
showed that 6 
out of the remaining 8 sentence pairs 
also involved 
out-of-date sense annotations.
Since we were unable to make a confident judgment on
the remaining two instances, 
we conclude that
the substitution test experiment
yields no clear exceptions to our theory.

\subsection{Automatic Error Correction}
\label{auto-err-cor}

\begin{figure}[t]
\begin{small}
\begin{algorithmic}[1]
    \State{Input: Sense alignment pairs 
      $(s_x, t_u)$ and $(s_y, t_v)$ 
      such that syn$(t_u, t_v)$}
    \If{NOT syn$(s_x, s_y)$}
      \State{$P \gets$ the multi-synset containing $t_u$ and $t_v$}
      
      \If{$w(s_x) \in P$}
        \State{$s'_x \gets (w(s_x), P)$}
        \State{CORRECT: $(s_x, t_u) \;\rightarrow\; (s'_x, t_u)$}
      \Else
        \State{ADD: $w(s_x)$ to $P$}
      \EndIf
      
      \If{$w(s_y) \in P$}
        \State{$s'_y \gets (w(s_y), P)$}
        \State{CORRECT: $(s_y, t_v) \;\rightarrow\; (s'_y, t_v)$}
      \Else
        \State{ADD: $w(s_y)$ to $P$}
      \EndIf
      
    \EndIf
\end{algorithmic}
\end{small}
\caption{Error correction algorithm.
$w(s)$ refers to the word of which $s$ is a sense.
$(d, S)$ is the sense of word $d$ in synset $S$.}
\label{pcode}
\end{figure}
 
The sense-level verification experiments
demonstrate that our theory can be applied to 
automatically {\em detecting} errors in sense-annotated corpora.
In this section,
we propose an algorithm for
{\em correcting} such errors,
which is also able to amend a corresponding multi-wordnet.
The algorithm is based on Theorem~\ref{four-senses},
which predicts that
any two pairs of aligned bitext senses
that are related by synonymy in either language
must all share the same multi-synset.

The pseudo-code of the error correction algorithm 
is shown in Figure \ref{pcode}.
The algorithm takes as input 
two sense alignment pairs,
and outputs a suggested error correction
for any exception to Theorem \ref{four-senses}.
When an exception to Theorem \ref{four-senses} is detected,
the algorithm either corrects the corresponding annotation in the bitext 
or
suggests a new sense to be added to the multi-wordnet.
The algorithm can be applied in either translation direction.

When applied to the English part of MSC and MWN,
the algorithm 
suggests 9028 sense corrections,
and 1166 sense additions.
We verified the suggestions on the sample of 50 manually-analyzed exceptions 
described in Section~\ref{analysis}.
We find that 
34 out of 39 proposed sense corrections
and 9 out of 13 proposed sense additions
are correct,
yielding an overall
accuracy of 83\%. 
We conclude that the algorithm could be effectively applied 
to automatically update sense annotations in a bitext,
and, more generally,
to correct errors and omissions in lexical resources.

\section{Concept Universality}
\label{cu}

We have demonstrated that word senses in wordnets
are objectively determined by 
the relation of near-synonymy between words.
Synsets are equivalence classes of synonymous senses,
which represent lexicalized concepts.
These concepts are discrete and disjoint.
Our theory does not contradict
the well-known thesis of \newcite{kilgarriff1997}
that word senses can only be defined 
relative to an intended application.
His critique,
which was formulated before 
WordNet's adoption as the standard WSD sense inventory,
is aimed at dictionary senses 
defined by lexicographers independently for each word.
In contrast, 
wordnet senses are grounded in the concept of synonymy,
and our theory is driven by multi-lingual applications, including translation.

Since senses are induced by near-synonymy relations between words,
we posit that
the number of senses in a wordnet cannot be substantially reduced 
without violating the synset properties 
formulated in Section \ref{synsets}.
In particular,
synset property~\#\ref{pShare} implies that  
each non-absolute synonym word pair
must involve 
multiple distinct word senses.
As a consequence,
the coarse-grained sense inventories created 
by clustering wordnet senses \cite{navigli2006,hovy2006}
cannot be assumed to preserve the synset properties.

The theorem and corollaries 
in Section~\ref{sense_theorems}
establish that 
all senses that are synonymous 
or translationally equivalent
share the the same multi-synset.
This implies a one-to-one mapping between synsets across languages,
with lexical gaps represented by empty synsets.
If we view a pair of wordnets as a bipartite graph
in which nodes are non-empty synsets
and edges represent the relation of translational equivalence,
then every node has a degree of at most one.
Since every synset represents a different lexicalized concept,
this implies that a concept in one language 
cannot correspond to more than one concept in another language.
We refer to this
implication of our theory as 
the {\em concept universality principle}.

Although there are many possible ways 
of organizing the semantic space into concepts,
the set of lexicalized concepts in a wordnet 
is deterministically induced 
by intra-lingual near-synonymy.
This is achieved by the objective application of 
the binary judgments of native speakers 
on the substitution test.
Because these sets of concepts differ between languages,
the creation of a multi-wordnet induces 
a common inter-lingual set of concepts,
determined by cross-lingual synonymy.
This in turn is based on the judgments of bilingual speakers
on the translational equivalence of pairs of senses.
The deterministically induced set of concepts 
in a hypothetical multi-wordnet encompassing all natural languages
could therefore be considered universal.
Not all of those concepts are necessarily lexicalized in any given language,
but each concept is lexicalized in at least one language.
Our theory implies that there is an exact matching between lexicalized concepts
across languages,
which satisfies the multi-wordnet assumption.

In practical terms,
the concept universality principle implies that
any differences in coverage between concepts across languages
must be resolved 
by increasing the granularity of the corresponding multi-wordnets.
For example, 
if one language makes a lexical distinction between 
``father's brother'' and ``mother's brother or aunt's husband'' 
and another language has different words for 
``father or mother's brother'' and ``aunt's husband'',
then all three of these concepts need to be represented 
by distinct synsets
in a multi-wordnet.\footnote{A similar argument can be made for
color terminology \cite{mccarthy2019}.}
This concept-adding approach is necessary to preserve the
multi-wordnet assumption,
which ensures that multi-synsets encode correct word translation pairs.
It also offers a theoretically-sound solution to the concerns
raised by \newcite{francopoulo2009} and \newcite{kwong2018}
regarding the expand model.

The concept universality principle,
which 
we have shown
to follow logically from the fundamental assumptions of wordnets,
provides theoretical support for 
constructing multi-wordnets using
the expand model, as opposed to the merge model 
(Section~\ref{multisynsets}).
A set of universal concepts in the expand model
provides a consistent level of granularity across languages,
as opposed to the variable level of granularity 
of individual wordnets in the merge model.
At the same time,
the principle provides a basis 
for avoiding bias towards English lexicalization patterns,
which has its roots in
the practice of founding new multi-wordnets 
on the synset structure of the original Princeton WordNet.
Because the expand model specifies no procedure for adding new synsets,
existing multi-wordnets such as BabelNet
are restricted to the set of concepts that was created for English.
We hope that 
the adoption of the universality principle 
will lead to 
the incorporation of conceptual distinctions
from other languages,
thus
guiding the evolution of multi-wordnets 
away from the hegemony of English,
and
toward greater linguistic diversity.

\section{Conclusion}
\label{conclusion}

We have proposed a 
unifying treatment of the notions of sense, synonymy and translational equivalence.
The resulting theory
formalizes the relationship between words and senses 
in both {\monolingual} and {\multilingual} settings.
In the future, we plan to investigate 
how our theory can best facilitate the task of
automating the
construction of
semantic resources. 
We expect that sound theoretical foundations will also lead to
improvements in both word sense disambiguation and machine translation.

\bibliographystyle{named}
\bibliography{wsd}

\end{document}